\documentclass{article}

\usepackage{microtype}
\usepackage{graphicx}
\usepackage{subfigure}
\usepackage{booktabs} 
\usepackage{color}
\usepackage{url}
\usepackage{times}
\usepackage{soul}
\usepackage{mathrsfs}
\usepackage{hhline}
\usepackage{bm}
\usepackage{dsfont}
\usepackage{tabularx}
\usepackage{multirow}
\usepackage{amsmath,amsfonts,amssymb}
\usepackage{color}
\usepackage{setspace}
\usepackage{float}
\usepackage{mathtools}
\usepackage{amsthm}
\newtheorem{theorem}{Theorem}

\usepackage{textcomp}

\usepackage{arydshln}

\newcommand{\nop}[1]{}

\usepackage[accepted]{icml2020}

\icmltitlerunning{Adversarial Mutual Information for Text Generation}

\begin{document}

\twocolumn[
\icmltitle{Adversarial Mutual Information for Text Generation}

\icmlsetsymbol{equal}{*}

\begin{icmlauthorlist}
\icmlauthor{Boyuan Pan}{1,equal}
\icmlauthor{Yazheng Yang}{2,equal}
\icmlauthor{Kaizhao Liang}{3}
\icmlauthor{Bhavya Kailkhura}{4}
\icmlauthor{Zhongming Jin}{5}
\icmlauthor{Xian-Sheng Hua}{5}
\icmlauthor{Deng Cai}{1}
\icmlauthor{Bo Li}{3}

\end{icmlauthorlist}

\icmlaffiliation{1}{State Key Lab of CAD\&CG, Zhejiang University}
\icmlaffiliation{2}{College of Computer Science, Zhejiang University}
\icmlaffiliation{3}{University of Illinois Urbana-Champaign}
\icmlaffiliation{4}{Lawrence Livermore National Laboratory}
\icmlaffiliation{5}{Alibaba Group}

\icmlcorrespondingauthor{Boyuan Pan}{panby@zju.edu.cn}
\icmlcorrespondingauthor{Bo Li}{lbo@illinois.edu}

\icmlkeywords{Machine Learning, ICML}

\vskip 0.3in
]



\printAffiliationsAndNotice{\icmlEqualContribution}

\begin{abstract}
    Recent advances in maximizing mutual information (MI) between the source and target have demonstrated its effectiveness in text generation. However, previous works paid little attention to modeling the backward network of MI (\textit{i.e.} dependency from the target to the source), which is crucial to the tightness of the variational information maximization lower bound. 
    In this paper, we propose \textit{Adversarial Mutual Information} (AMI): a text generation framework which is formed as a novel saddle point (min-max) optimization aiming to identify joint interactions between the source and target. Within this framework, the forward and backward networks are able to iteratively promote or demote each other's generated instances by comparing the real and synthetic data distributions. We also develop a latent noise sampling strategy that leverages random variations at the high-level semantic space to enhance the long term dependency in the generation process. Extensive experiments based on different text generation tasks demonstrate that the proposed AMI framework can significantly outperform several strong baselines, and we also show that AMI has potential to lead to a tighter lower bound of maximum mutual information for the variational information maximization problem.
\end{abstract}

\section{Introduction}
Generating diverse and meaningful text is one of the coveted goals in machine learning research. Most sequence transductive models for text generation can be efficiently trained by maximum likelihood estimation (MLE) and have demonstrated dominant performance in various tasks, such as dialog generation~\cite{serban2016building,park2018hierarchical}, machine translation~\cite{bahdanau2014neural,vaswani2017attention} and document summarization~\cite{see2017get}. However, the MLE-based training schemes have been shown to produce safe but dull texts which are ascribed to the relative frequency of generic phrases in the dataset such as ``\textit{I don't know}" or ``\textit{what do you do}"~\cite{serban2016building}.

Recently, \citet{li2016diversity} proposed to use maximum mutual information (MMI) in the dialog generation as the objective function. They showed that modeling mutual information between the source and target will significantly decrease the chance of generating bland sentences and improve the informativeness of the generated responses, which has inspired many important works on MI-prompting text generation~\cite{li2016deep,zhang2018generating,ye2019jointly}. In practice, directly maximizing the mutual information is intractable, so it is often converted to a lower bound called the variational information maximization (VIM) problem by defining an auxiliary backward network to approximate the posterior of the dependency from the target to the source~\cite{chen2016infogan}. Specifically, the likelihood of generating the real source text by the backward network, whose input is the synthetic target text produced by the forward network, will be used as a reward to optimize the forward network, tailoring it to generate sentences with higher rewards. 

However, most of them optimize the backward network by maximizing the likelihood stated above. Without the guarantee of the quality of the synthetic target text produced by the forward network, this procedure encourages the backward model to generate the real source text from many uninformative sentences and is not likely to make the backward network approach the true posterior, thus affects the tightness of the variational lower bound. On the other hand, since the likelihood of the backward network is a reward for optimizing the forward network, a biased backward network may provide unreliable reward scores, which will mislead the forward network to still generate bland text.

In this paper, we present a simple yet effective framework based on the maximum mutual information objective that encourages it to train the forward network and the backward network adversarially. We propose \textit{Adversarial Mutual Information} (AMI), a new text generative framework that addresses the above issues by solving an information-regularized minimax optimization problem. Instead of maximizing the variational lower bound for both the forward network and the backward network, we maximize the objective when training the forward network while minimizing the objective when training the backward network. In addition to the original objective, we add the expectation of the negative likelihood of the backward model whose sources and targets are sampled from the real data distribution. To this end, we encourage the backward network to generate real source text only when its input is in the real target distribution. This scheme remedies the problem that the backward model gives high rewards when its input is the bland text or the text that looks human-generated but has little related content to the real data. To stabilize training, we prove that our objective function can be transformed to minimizing the Wasserstein Distance~\cite{arjovsky2017wasserstein}, which is known to be differentiable almost everywhere under mild assumptions. Moreover, we present a latent noise sampling strategy that adds sampled noise to the high-level latent space of the model in order to affect the long-term dependency of the sequence, hence enables the model to generate more diverse and informative text.

{\bf \underline{Technical Contributions}}
In this paper, we take the first step
towards generating diverse and informative text by optimizing adversarial mutual information. We make contributions on both the theoretical
and empirical fronts.
\vspace{-0.5em}
\begin{itemize}
    \item  We propose \textit{Adversarial Mutual Information} (AMI), a novel text generation framework that adversarially optimizes the mutual information maximization objective function. We provide a theoretical analysis on how our objective can be transformed into the Wasserstein distance minimization problem under certain modifications.
	\item To encourage generating diverse and meaningful text, we develop a latent noise sampling method that adds random variations to the latent semantic space to enhance the long term dependency for the generation process.
	\item We conduct extensive experiments on the tasks of dialog generation and machine translation, which demonstrate that the proposed AMI framework outperforms several strong baselines significantly. We show that our method has the potential to lead to a tighter lower bound of maximum mutual information.
\vspace{-0.5em}
\end{itemize}

\section{Background}
\subsection{Sequence Transduction Model}
\label{seq}
The sequence-to-sequence models~\cite{cho2014learning,sutskever2014sequence} and the Transformer-based models~\cite{vaswani2017attention} have been firmly established as the state-of-the-art approaches in text generation problems such as dialog generation or machine translation~\cite{wu2018globaltolocal,pan2018macnet,akoury2019syntactically}. In general, most of these sequence transduction models hold an encoder-decoder architecture~\cite{bahdanau2014neural}, where the encoder maps an input sequence of symbol representations $S = \{x_1, x_2, ..., x_n\}$ to a sequence of continuous representations $\bm{z} = \{z_1, z_2, ..., z_n\}$. Given $\bm{z}$, the decoder then generates an output sequence $T = \{y_1, y_2, ..., y_m\}$ of tokens one element at a time. At each time step the model is auto-regressive, consuming the previously generated tokens as additional input when generating the next. 

For the sequence-to-sequence models, the encoder and the decoder are usually based on recurrent or convolutional neural networks, and the Transformer builds them using stacked self-attention and point-wise fully connected layers. Given the source $S$ and the target $T$, the objective of a sequence transduction model can be simply formulated as:
\begin{equation}
\begin{aligned}
\theta^{*} = \arg\max_{\theta} ~P_{\theta}(T|S)\\
\end{aligned}
\end{equation}
where $\theta$ is the parameters of the sequence transduction model.

\begin{figure*}[t]
	\center
	\includegraphics[width=0.95 \textwidth]{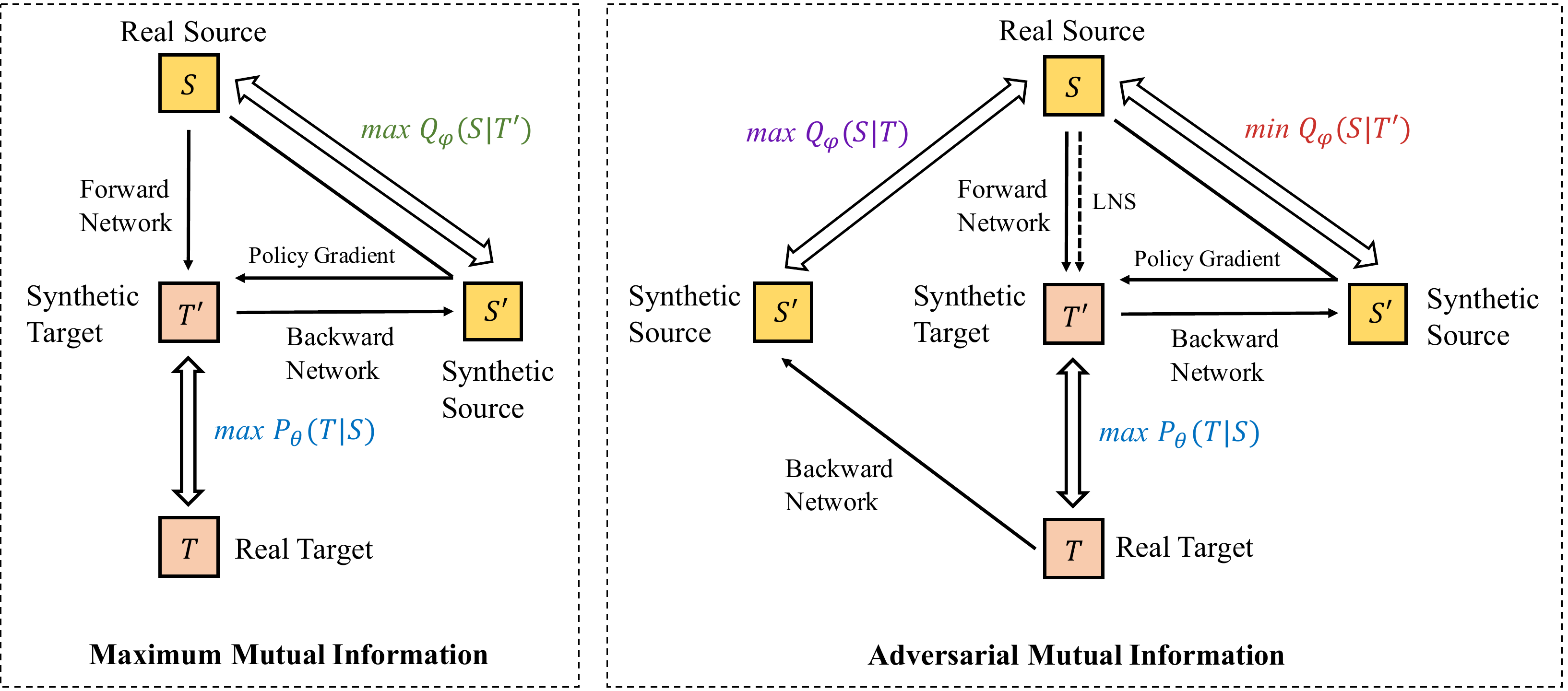}
	\caption{\label{fig1}The overview of frameworks for the Maximum Mutual Information (left part) and the proposed Adversarial Mutual Information (right part). Different colors of text denotes different objectives.}
\end{figure*}

\subsection{Maximum Mutual Information}
Maximum mutual information for text generation is increasingly being studied~\cite{li2016diversity,zhang2018generating,ye2019jointly}. Compared to the maximum likelihood estimation (MLE) objective, maximizing the mutual information $I(S,T)$ between the source and target encourages the model to generate sequences that are more specific to the source. Formally, we have
\begin{equation}
\begin{aligned}
I(S,T) = \mathds{E}_{P(T,S)} \left[{\rm log}\frac{P(S,T)}{P(S)P(T)}\right].
\end{aligned}
\end{equation}
However, directly optimizing the mutual information is intractable. To provide a principled approach to maximizing MI, a Variational Information Maximization~\cite{barber2003algorithm,chen2016infogan,zhang2018generating} lower bound is adopted:
\begin{equation}
\begin{aligned}
\label{elbo}
I(S,T) &=  H(S) - H(S|T)\\
&= H(S) + \mathds{E}_{P(T,S)} \left[{\rm log} P(S|T)\right]\\
&= H(S) + \mathds{E}_{P(T)} \left[D_{KL}(P(S|T) || Q_{\phi}(S|T))\right] \\
&+ \mathds{E}_{P(T,S)} \left[{\rm log} Q_{\phi}(S|T)\right]\\
& \geq H(S) + \mathds{E}_{P(S)} \mathds{E}_{P_{\theta}(T|S)} \left[{\rm log}Q_{\phi}(S|T)\right]
\end{aligned}
\end{equation}
where $H(\cdot)$ denotes the entropy, and $D_{KL}(\cdot, \cdot)$ denotes the KL divergence between two distributions. $Q_{\phi}(S|T)$ is a \textit{backward network} that approximates the unknown $P(S|T)$, and usually shares the same architecture with the forward network $P_{\theta}(T|S)$. Note that the entropy term associated with the training data is ignored as it does not involve the parameters we are trying to optimize.

Given the source text $S$, the objective of MMI is equivalent to maximizing:
\begin{equation}
\begin{aligned}
\label{eq_mmi}
\max_{\theta, \phi} ~\mathds{E}_{T' \sim P_{\theta}(T'|S)} \left[{\rm log}Q_{\phi}(S|T')\right]
\end{aligned}
\end{equation}
where $\theta$ and $\phi$ are the parameters of the forward network and the backward network respectively. We denote the synthetic target text generated by the forward network as $T'$ in order to distinguish it from the real data.

\section{Adversarial Mutual Information}
In this section, as shown in the Figure \ref{fig1}, we present the drawbacks of directly optimizing the MMI problem and introduce how our proposed Adversarial Mutual Information (AMI) solves these issues by adversarially optimizing its modules.

Intuitively, the objective in the equation (\ref{eq_mmi}) for the forward network $P_{\theta}$ can be viewed as to use the log-likelihood of the backward model $Q_{\phi}$, whose input is the synthetic target text produced by $P_{\theta}$, to reward or punish $P_{\theta}$. This encourages the forward network to generate text that has a stronger connection with its input. For the backward network $Q_{\phi}$, however, maximizing the log-likelihood of $Q_{\phi}$ when its input is the synthetic text is not always optimal, especially when this synthetic text is uninformative or has little intersection in content with the source text. Although this scheme is maximizing the mutual information in the form, it optimizes $Q_{\phi}$ towards a biased direction and a low-quality $Q_{\phi}$ may bring unreliable reward signal. In addition, the lower bound shown in the equation (\ref{elbo}) becomes tight as the approximated distribution $Q_{\phi}$ approaches the true posterior distribution: $\mathds{E} \left[D_{KL}(P(\cdot) || Q_{\phi}(\cdot))\right] \rightarrow 0$, which is not likely to be achieved when many negative samples are involved. 

Inspired by the delicate but straightforward idea of the generative adversarial networks~\cite{goodfellow2014generative}, we set all the synthetic target text as the negative samples and force the backward model to generate correct source text only if its input is the real target text. Formally, we propose to solve the following minimax game:
\begin{equation}
\label{eq_adv}
\small
\begin{aligned}
\min_{\phi}\max_{\theta} ~ 
\mathrlap{\overbrace{\phantom{P_{\theta}(T|S) + \mathds{E}_{T' \sim P_{\theta}(T'|S)} \left[Q_{\phi}(S|T')\right]}}^{\text{\scriptsize Mutual Information}}}
      P_{\theta}(T|S) + 
      \mathrlap{\underbrace{\phantom{\mathds{E}_{T' \sim P_{\theta}(T'|S)} \left[Q_{\phi}(S|T')\right] - (Q_{\phi}(S|T)}}_{\text{\scriptsize Adversarial Training}}}
      \mathds{E}_{T' \sim P_{\theta}(T'|S)} \left[Q_{\phi}(S|T')\right] 
      - Q_{\phi}(S|T)
\end{aligned}
\end{equation}
As we can see, we still maximize the mutual information when optimizing $P_{\theta}$, but we minimize it when optimizing $Q_{\phi}$. Moreover, we add the third term which minimizes the negative likelihood of $Q_{\phi}$ when feeding it with the real data. By regularizing the backward network with both the synthetic data and real data, the forward network will be trained against a smarter backward network that only rewards informative and related text. Here we omit the logarithm symbol for simplicity, and we also add $P_{\theta}(T|S)$ to the objective in order to balance the positive samples, which is referred as the teacher-forcing algorithm~\cite{li2017adversarial}.

\subsection{Wasserstein Distance Minimization}
Unfortunately, training such a framework similar to GAN is well known of being delicate and unstable~\cite{DBLP:conf/iclr/ArjovskyB17}. The recent work of~\cite{arjovsky2017wasserstein} indicated the superiority of the \textit{Wasserstein distance} in learning distributions due to its continuity and differentiability. They demonstrated their method significantly improves the stability of training and gained significant performance with the objective of minimizing the Wasserstein distance. In the following, we will prove that our adversarial training part of the objective function can be transformed to minimizing the Wasserstein distance under some simple assumptions.

To learn the proper distribution of the dialog utterances, in practice we replace the third term of the equation (\ref{eq_adv}) by the expectation $\mathds{E}_{T \sim P_{r}(D)} Q_{\phi}(S|T)$ where $P_r (D)$ is the real data distribution, $D$ is the dataset. 
\begin{theorem} 
Let $P_{r}(D)$ be any distribution over $D$. Let $P_{\theta}(\cdot | S)$ be the distribution of a neural network given $S$.  $Q_{\phi}(S|T)$ represents the probability of a neural network where the input is $T$ and the output is $S$, which can be formed as a function $f_{\phi}(S,T)$. We say the following problem can be considered as the Wasserstein Distance between $P_{r}$ and $P_{\theta}$ if there is a constant $K$ satisfying $\lVert f_{\phi} \lVert_{L} \leq K$:
\begin{equation}
\begin{aligned}
\min_{\phi}~ \mathds{E}_{T' \sim P_{\theta}(T'|S)} \left[Q_{\phi}(S|T')\right] -\mathds{E}_{T \sim P_{r}(D)} \left[Q_{\phi}(S|T)\right]
\end{aligned}
\end{equation}
where $\lVert \cdot \lVert_{L} \leq K$ denotes K-Lipschitz.
\end{theorem}
\begin{proof}
It is obvious that the above problem is equivalent to:
\begin{equation}
\begin{aligned}
\max_{\phi}~ \mathds{E}_{T \sim P_{r}(D)} \left[Q_{\phi}(S|T)\right] - \mathds{E}_{T' \sim P_{\theta}(T'|S)}\left[Q_{\phi}(S|T')\right]
\end{aligned}
\end{equation}
which is also equivalent to solving the problem:
\begin{equation}
\begin{aligned}
&\sup_{ \lVert f_{\phi} \lVert_{L} \leq K}~ \mathds{E}_{T \sim P_{r}(D)}\left[ Q_{\phi}(S|T)\right] - \mathds{E}_{T' \sim P_{\theta}(T'|S)} \left[Q_{\phi}(S|T')\right]
\end{aligned}
\end{equation}
where supremum is over all the $K$-Lipschitz functions $Q_{\phi}$. According to the Kantorovich-Rubinstein duality~\cite{villani2008optimal}, the above form is $K$ times of the Wasserstein Distance between $P_{r}$ and $P_{\theta}$:
\begin{equation}
\begin{aligned}
W(P_{r},  P_{\theta}) = \inf_{\gamma \in \prod (P_{r},  P_{\theta})}\mathds{E}_{(x,y) \sim \gamma} \left[ \lVert x - y \lVert \right]
\end{aligned}
\end{equation}
where $\prod (P_{r},  P_{\theta})$ denotes the set of all joint distributions $\gamma(x, y)$ whose marginals are respectively $P_{r}$ and $P_{\theta}$.
\end{proof}
In order to satisfy the $K$-Lipschitz striction of $f_{\phi}$, we use the weight clipping trick~\cite{arjovsky2017wasserstein} to clamp the weights to a fixed box (say $[-0.01, 0.01^{l}]$) after each gradient update. Therefore, we modify the adversarial training part in the equation (\ref{eq_adv}) as:
\begin{equation}
\begin{aligned}
\label{eq_wgan}
\min_{\theta}\max_{\phi}~ \mathds{E}_{T \sim P_{r}(D)} \left[Q_{\phi}(S|T)\right] - \mathds{E}_{T' \sim P_{\theta}(T'|S)} \left[Q_{\phi}(S|T')\right]
\end{aligned}
\end{equation}
which can be transformed to the problem of minimizing a Wasserstein Distance.

We notice that sometimes the forward network may generate $T'$ that is the same as $T$ or very similar, which should not be punished for the adversarial training of the backward model. Hence, we add a multiplier $K(T')$ to the second term of the equation (\ref{eq_wgan}) to regularize the minimization procedure:
\begin{equation}
\begin{aligned}
\label{eq_k}
& -\mathds{E}_{T' \sim P_{\theta}(T'|S)} \left[K(T') \cdot Q_{\phi}(S|T')\right] \\
= & - \mathds{E}_{T' \sim P_{\theta}(T'|S)} \left[(1 - cos(T,T')) \cdot Q_{\phi}(S|T')\right] \\
\end{aligned}
\end{equation}
where $cos(T,T')$ is the cosine similarity between the embeddings of $T$ and $T'$.

Since the equation (\ref{eq_k}) w.r.t. $\theta$ is not differentiable, we adopt Monte Carlo samples using the REINFORCE policy~\cite{williams1992simple} to approximate the gradient with regard to $\theta$:
\begin{equation}
\begin{aligned}
&-\nabla_{\theta} \mathds{E}_{T' \sim P_{\theta}(T'|S)} \left[K(T') \cdot Q_{\phi}(S|T')\right]\\
= &-\mathds{E}_{T' \sim P_{\theta}(T'|S)} \left[K(T') \cdot Q_{\phi}(S|T') - b \right] \cdot \nabla_{\theta}{\rm log}P_{\theta}(T'|S),\\
\end{aligned}
\end{equation}
To remedy the high variance of learning signals, we adopt the baseline function $b$ by empirically averaging the signals to stabilize the learning process~\cite{williams1992simple}. The gradient of the second term of the equation (\ref{eq_k}) with regard to $\phi$ can be calculated as: 
\begin{equation}
\begin{aligned}
& - \nabla_{\phi} \mathds{E}_{T' \sim P_{\theta}(T'|S)} \left[K(T') \cdot Q_{\phi}(S|T')\right] \\
=&- \mathds{E}_{T' \sim P_{\theta}(T'|S)} \left[K(T') \cdot \nabla_{\phi}Q_{\phi}(S|T')\right]
\end{aligned}
\end{equation}

\begin{table*}[t]
    \small
	\centering

	\begin{tabular}{p{4.5cm}ccc|cccc}
		\toprule
		\multicolumn{1}{c}{\multirow{2}{*}{\textbf{Models}}} & \multicolumn{3}{c}{\textbf{Relevance}} & \multicolumn{3}{c}{\textbf{Diversity}} &  \\ \cline{2-7} 
		\multicolumn{1}{c}{}  & \textbf{Average}  & \textbf{Greedy} & \textbf{Extrema}    & \textbf{~Dist-1~}     & \textbf{~Dist-2~} &\textbf{~Ent-4~}     \\ \hline 
        HRED~\cite{serban2016building} & 0.820 & 0.623 & 0.391 & 0.063 & 0.262 & 8.861 \\
        cGAN~\cite{li2017adversarial} & 0.841 & 0.648 & 0.410  & 0.069 &  0.272 &  9.418  \\ 
        VHCR~\cite{park2018hierarchical} & 0.855 & 0.663 & 0.430  & 0.085 & 0.348 & 9.602\\
        Dir-VHRED~\cite{zeng2019dirichlet}  & 0.862 & 0.660  & 0.431  & 0.076 &  0.306 &  9.586  \\ \hline \hline
        LSTM  & 0.817 & 0.611 & 0.383 & 0.061 & 0.260 & 8.855\\ 
        LSTM + MMI & 0.825 & 0.617 & 0.391 & 0.065 & 0.265 & 8.894\\ 
        LSTM + \textbf{AMI} (w/o LNS)& 0.851 & 0.644 & 0.423 & 0.078 & 0.287 & 9.176 \\
        LSTM + \textbf{AMI} & 0.858 & 0.655 & 0.429 & 0.082 & 0.294 & 9.205 \\ \hline
        Transformer & 0.828 & 0.627 & 0.398  & 0.077 & 0.347 & 9.440 \\ 
        Transformer + MMI & 0.835 & 0.636 & 0.406 & 0.082 & 0.356 & 9.493  \\  
        Transformer + \textbf{AMI}  (w/o LNS) & 0.859 & 0.661 & 0.424  & 0.096  & 0.372   & 9.632 \\ 
        Transformer + \textbf{AMI}  & \textbf{0.869} & \textbf{0.672}  & \textbf{0.433} & \textbf{0.104}  & \textbf{0.385} & \textbf{9.695} \\ 
		\bottomrule
	\end{tabular}
\caption{\label{tab1} Quantitative evaluation for dialog generation on the PersonaChat dataset. ``LNS" denotes the latent noise sampling. The top part presents the baselines, the medium and bottom parts show the performance and ablation results of our AMI framework based on the LSTM and Transformer. Our AMI significantly improves both LSTM and Transformer and achieves the state-of-the-art results.}  
\end{table*}

\subsection{Latent Noise Sampling}
For the Monte Carlo sampling approaches in the policy gradient for NLP tasks, existing works often employ low-level sampling methods such as greedy search or beam search~\cite{li2016deep,paulus2018a}. However, methods like beam search usually produce very similar or even identical results, thus require a large beam width and become inefficient and time-consuming. 

Recent studies in the variational autoencoder (VAE) frameworks have demonstrated that sampling from high-level space can generate abundant variability of natural language by capturing its global and long-term structure~\cite{serban2017hierarchical,zeng2019dirichlet}. Nevertheless, a known problem of the VAE-based dialog model is \textit{degeneracy}, which ignores the latent variable and causes model to behave as a vanilla sequence-to-sequence model~\cite{bowman2016generating,park2018hierarchical}. Moreover, instead of modeling the relation between the given source and target pairs, we expect our samples to be diverse and not limited by the knowledge of the target label. Thus we sample the target text from $P_{\theta}(T'|S)$ by simply adding a noise to the latent representations (\textsection \ref{seq}) of the encoder:
\begin{equation}
\begin{aligned}
\tilde{\bm{z}} = \bm{z} + \delta.
\end{aligned}
\end{equation}
Since the goal of the noise is to maximize both (1) the diversity the generated $T'$ and (2) the probability of reconstructing $S$ with the generated $T'$, we optimize it by:
\begin{equation}
\begin{aligned}
\label{eq_lns}
\max_{\delta} \lambda \lVert \delta \lVert + \mathds{E}_{T' \sim P_{\theta}(T'|S, \delta)} \left[Q_{\phi}(S|T')\right]
\end{aligned}
\end{equation}
where $\lambda$ is a scaling factor accounting for the difference in magnitude between the above two terms. The noise $\delta$ is computed by:
\begin{gather}
\delta \sim \mathcal{N} (\bm{\mu}(\bm{z}), \bm{\sigma}^2(\bm{z}))\\
\bm{\mu}(\bm{z}) = \bm{W}_2 \cdot{\rm ReLU}(\bm{W}_1\bm{z} + \bm{b}_1)\\
\bm{\sigma}(\bm{z}) = {\rm Softplus(\bm{\mu}(\bm{z}) )}
\end{gather}
where $\mathcal{N}(\cdot, \cdot)$ is a Gaussian distribution. Since the backpropagation can not flow through a random node, in practice, we use the reparameterization trick~\cite{kingma2014auto}:
\begin{equation}
\begin{aligned}
\delta &= \bm{\mu}(\bm{z}) + \epsilon \cdot \bm{\sigma}(\bm{z})
\end{aligned}
\end{equation}
where the sample $\epsilon \sim \mathcal{N}(0,\bm{I})$. Therefore, the final AMI objective function is:
\begin{equation}
\begin{aligned}
& \min_{\theta}\max_{\phi} ~ \mathds{E}_{T \sim P_{r}(D)} \left[Q_{\phi}(S|T)\right]\\
&- \mathds{E}_{T' \sim P_{\theta}(T'|S, \delta)} \left[K(T') \cdot Q_{\phi}(S|T')\right]- P_{\theta}(T|S) 
\end{aligned}
\end{equation}
where $T' \sim P_{\theta}(T'|S, \delta)$ means sampling the target text $T'$ from $P_{\theta}(T'|S)$ by adding the noise $\delta$ in the high-level latent space of the model. Note that for the RNN-based sequence-to-sequence model, $\bm{z}$ denotes the concatenation of the last hidden states from both directions.

\section{Experiments}
To show the effectiveness of our approach, we use the bi-directional LSTM and Transformer~\cite{vaswani2017attention} as our base architectures. We apply our AMI framework to the dialog generation (\textsection \ref{dialog}) and neural machine translation (\textsection \ref{nmt}), which are representatives of two popular text generation tasks, and conduct comprehensive analyses on them. Code is available at \url{https://github.com/ZJULearning/AMI}.

\subsection{Dialog Generation}
\label{dialog}
We first verify the effectiveness of our method on the dialog generation task, which requires to generate a coherent and meaningful response given a conversation history.

\paragraph{Dataset}
We evaluate our dialog model on the PersonaChat dataset\footnote{\url{https://github.com/facebookresearch/ParlAI/tree/master/projects/personachat}}~\cite{zhang2018personalizing}. The dataset consists of conversations between crowdworkers who were randomly paired and asked to act the part of a given persona and chat naturally. There are around 160,000 utterances in around 11,000 dialogues, with 2000 dialogues for validation and test, which use non-overlapping personas.

\paragraph{Implementation Details}
We pre-train the forward network and backward network, and iteratively optimize the noise $\delta$, forward network and backward network during AMI training. For the backward network, we follow~\cite{li2016deep} to predict the last utterance of the dialog history. Detailed configurations are in the Appendix \ref{app_dialog}. To measure the relevance between the generated utterance and the ground-truth one, we use the three categories of word-embedding metrics~\cite{rus-lintean-2012-comparison,serban2017hierarchical}: \textit{average}, \textit{greedy} and \textit{extrema}. The \textit{average} metric calculates sentence-level mean embeddings, while the \textit{greedy} and \textit{extrema} compute the word-to-word cosine similarity. Their difference lies on that \textit{greedy} takes the average word vector in the sentence as the sentence embedding while \textit{extrema} adopts the extreme value of these word vectors. To evaluate the diversity of the generated questions, we follow \cite{li2016diversity} to calculate \textit{Dist-n} (\textit{n=1,2}), which is the proportion of unique n-grams over the total number of n-grams in the generated utterances for all conversations, and \cite{zhang2018generating} to use the \textit{Ent-4} metric, which reflects how evenly the n-gram distribution is over all generated utterances.

\paragraph{Baselines}
We compare our proposed method based on the BiLSTM and Transformer with the following baselines: (i) HRED~\cite{serban2017hierarchical}: The hierarchical recurrent encoder-decoder baseline model. (ii) cGAN~\cite{li2017adversarial}: Conditional generative adversarial network based on the sequence-to-sequence model. (iii) VHCR~\cite{park2018hierarchical}: Variational hierarchical conversation RNN, a hierarchical latent variable model with the utterance drop regularization. (iv) Dir-VHRED~\cite{zeng2019dirichlet}: Using Dirichlet distribution in place of traditional Gaussian distribution in variational HRED for dialogue generation.

\paragraph{Results}

\begin{table}
\small
	\begin{center}

		\begin{tabular}{lcc}
			\toprule
			\textbf{Backward Network} &  \textbf{Synthetic}&  \textbf{Real} \\
			\midrule
            MMI & 0.807  & 0.805 \\ 
            MMI (fixed $\phi$) & 0.809 & 0.812 \\ 
            AMI  &  0.776 & 0.820 \\ 
			\bottomrule
		\end{tabular}
	\end{center}
		\caption{\label{tab22} Embedding average for the output of the backward network, where the input is the utterance (1) generated by the forward network or (2) from the dataset created by human.}    	

\end{table}

\begin{figure}[t]
	\center
	\includegraphics[width=0.4 \textwidth]{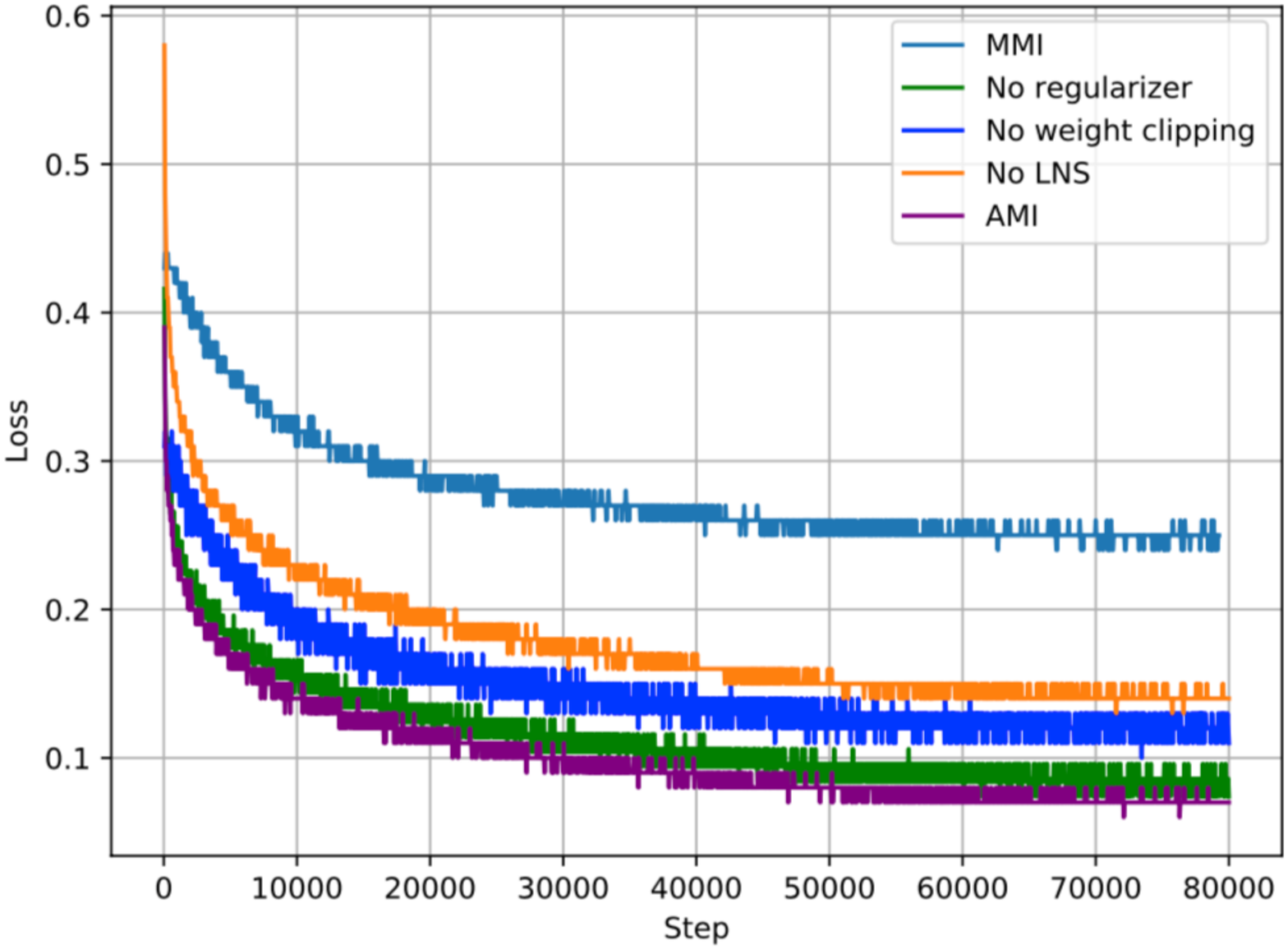}
	\caption{\label{fig2}Convergence curves for different ablated models. ``regularizer" denotes the cosine multiplier in the equation (\ref{eq_k}).}
	\vskip -0.1in
\end{figure}

In Table \ref{tab1}, we apply our AMI framework to LSTM and Transformer, and compare them with other competitive baseline models on the PersonaChat dataset. As we can see, our method with Transformer clearly outperforms all the baselines and achieves state-of-the-art results. From the medium and bottom parts, we observe that our AMI greatly improves the relevance and diversity performances over both the LSTM and Transformer. Compared to MMI, our method significantly exceeds on both the relevance and diversity metrics, which indicates that our adversarial training scheme helps improve the quality of the backward model, which subsequently provides more credible rewards to the forward network. We also conduct ablation experiments that when we use beam search for sampling instead of our proposed latent sampling noise, we find that the results drop on both metrics, which demonstrates that this sampling strategy can enlarge the searching space from the high-level structure while maximizing the mutual information between the output and the source text.

As shown in Table \ref{tab22}, we compare the performance of the backward network before and after training by our AMI framework. We evaluate the embedding average of the outputs of the backward network when feeding it with the synthetic data and the real data. Compared to the MMI with fixed $\phi$ (\textit{i.e.} the pre-trained model), the performance of the regular MMI backward model drops and its output text is more relevant to ground truth when fed with the synthetic data rather than the real data. This means such a training framework misleads the backward model to give lower rewards to the generated targets that look more like the real ones, and thus results in negative effect to the optimization of the forward model. With our adversarial training scheme, the backward model performs better when fed with the real data and worse with the synthetic data, which means it improves the ability to manage the information flow.

We also present the convergence curves over different ablated models of the AMI for LSTM in Figure \ref{fig2}. We can see that the regularizer $K(T')$ (equation (\ref{eq_k})) helps accelerate the convergence of the model. We also observe that without the weight clipping, which is important in transforming our objective to minimizing the Wasserstein distance, the training becomes more unstable and the loss can not approach the performance of the full AMI.

\paragraph{Human Evaluation}
\begin{table}
\small
	\begin{center}

		\begin{tabular}{lccc}
			\toprule
			\textbf{Opponent Models} &  \textbf{Wins}&  \textbf{Losses} & \textbf{Ties}\\
			\midrule
            LSTM + \textbf{AMI} \textit{vs.} LSTM + MMI & \textbf{0.62} & 0.24 & 0.14\\ 
            TF + \textbf{AMI} \textit{vs.} TF + MMI  & \textbf{0.54} & 0.28 & 0.18\\ 
            TF + \textbf{AMI} \textit{vs.} Dir-VHRED & \textbf{0.36} & 0.30 & 0.34 \\ 
            Dir-VHRED \textit{vs.} Human & 0.28 & \textbf{0.56} & 0.16 \\ 
            LSTM + \textbf{AMI} \textit{vs.} Human & 0.24 & \textbf{0.54} & 0.22 \\ 
            TF + \textbf{AMI} \textit{vs.} Human & 0.32 & \textbf{0.48} & 0.20 \\ 
			\bottomrule
		\end{tabular}
	\end{center}
		\caption{\label{tab3} Human evaluation results for the dialog generation. ``TF" means the Transformer, ``Human" means the original human-created dialog utterances in the dataset. Both ``Wins" and ``Losses" point to the left models.} 
\end{table}

We conduct the human evaluation to measure the quality of the generated dialog utterances, as shown in Table \ref{tab3}. Models are paired as opponents and we randomly presented 120 dialog history and the outputs of each paired models to 5 judges, who are asked to decide which of the two outputs is better. Ties are permitted. From the human evaluated results, our AMI is proven to be better at generating a coherent and meaningful dialog response than MMI, and our AMI on Transformer performs better than Dir-VHRED, one of the state-of-the-art dialog models. However, all of the synthetic utterances lost in the comparisons with the human created sentences, which indicates that the current generative models still have much room for improvement in understanding and generating natural language.

\paragraph{Bound Analysis}
\begin{figure}[t]
	\center
	\includegraphics[width=0.4 \textwidth]{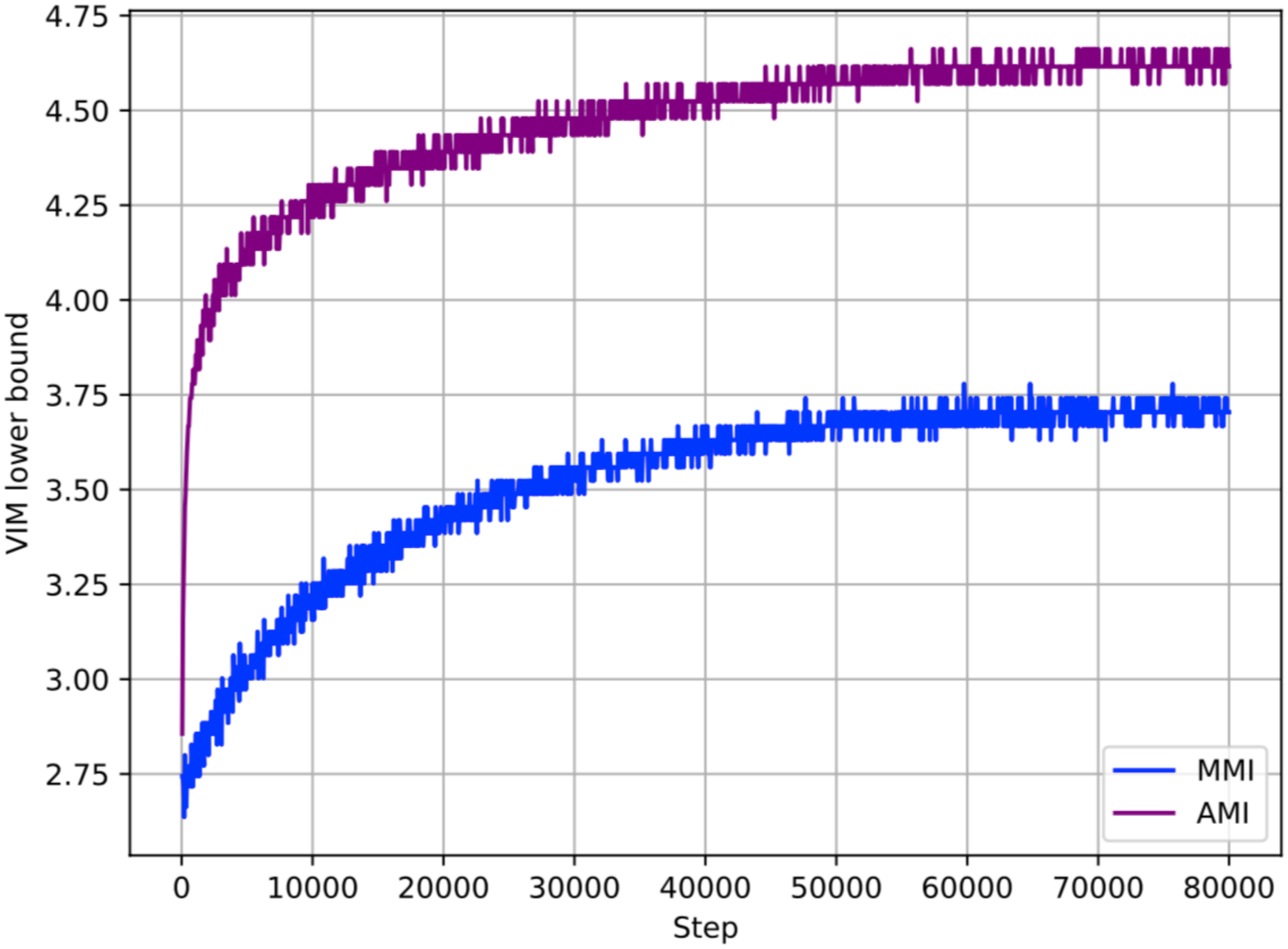}
	\caption{\label{fig3}Variational lower bound over training steps.}
	\vskip -0.1in
\end{figure}

To evaluate whether the mutual information between the source and target text can be efficiently maximized, or whether our method can lead to a tighter variational lower bound, we compare the variational information maximization (VIM) lower bounds over MMI and our proposed AMI framework, both in Transformer architectures. We see from Figure \ref{fig3} that the lower bound in AMI is higher than in MMI. According to the equation (\ref{elbo}), if the backward models $Q_{\phi}$ of AMI and MMI get the same KL-divergence value, then the AMI obtains a higher mutual information between the source and target than MMI; On the other hand, if AMI and MMI get the same mutual information, the AMI obtains a lower KL-divergence of the backward model than the MMI, which means the backward network in the AMI is more closed to the posterior distribution, thus the variational lower bound becomes tighter.

\paragraph{Case Study}

\begin{table}[t]
    \setlength{\abovecaptionskip}{0pt}
    \setlength{\belowcaptionskip}{-2pt}
	\begin{center}
		\begin{tabular}{p{0.95\columnwidth}}
			\toprule
			{\small [A]: Oh , we had a graduation party from college for me. }\\
			{\small [B]: Congratulations ! What is your degree in? } \\ 
			{\small \textbf{Transformer: }I am studying to become a fashion designer.}\\ 
			{\small \textbf{Transformer + MMI: }I am studying law. What is your favorite color?} \\ 
			{\small \textbf{Transformer + AMI: }I am majoring in accounting and minoring in business management. What about you?} \\
			{\small \textbf{Human: }Math. Also had a wedding reception for my brother.}\\ \hline \hline
			{\small [A]: What is your novel about? I am a recent college graduate.}\\
			{\small [B]: It is actually a romance novel, but it has sci-fi elements.} \\ 
			{\small \textbf{Transformer: }What do you do for a living?}\\ 
			{\small \textbf{Transformer + MMI: }I do not really have a favorite. I do not mind.} \\ 
			{\small \textbf{Transformer + AMI: }Oh, I see. I am a huge fan of romance books.} \\
			{\small \textbf{Human: }That sounds interesting, I would like to read it some day.}\\
			\bottomrule
		\end{tabular}
	\end{center}
		\caption{\label{tab4} Examples of dialog generation. ``Human" denotes the original response in the dataset, and all the generation models represent the role [A].}   
\end{table}

In Table \ref{tab4}, we show two examples of the generated dialog responses along with the dialog history. In the first case, the MMI model tries to improve communication by adding a question. However, the added question ``what is your favorite color" has nearly no connection with the previous dialog history. Our method provides a more informative answer to the question of the role [B], and raises a very reasonable question in this situation. In the second case, the Transformer baseline produces a very common question, which can just carry the conversation but is not coherent and has little mutual information with the previous dialog history. The MMI provides a response that has more connection with the dialog history but is less suitable for this conversation since no one asked about his/her preference. Our method, however, generates a response that is not only informative but also coherent to the previous dialog. 

\subsection{Neural Machine Translation}
\label{nmt}
We also evaluate our method on the neural machine translation (NMT) task, which requires to encode a source language sentence and predict a target language sentence.

\paragraph{Dataset}
The dataset for our evaluation is the WMT translation task between English and German in both directions and the translation performances are reported on the official test set \texttt{newstest2014}\footnote{\url{http://www.statmt.org/wmt14/translation-task.html}}. The dataset contains 4.5M training pairs, 2169 validation pairs and 3003 test pairs.

\paragraph{Implementation Details}
The training scheme for each module is the same as the dialog generation task. We use BLEU~\cite{papineni2002bleu} as the evaluation metric for translation quality. BLEU is the geometric mean of $n$-gram precisions where $n \in \{1,2,3,4\}$ weighted by sentence lengths. Detailed configurations are in the Appendix \ref{app_nmt}.

\paragraph{Results}

\begin{table}
\small
	\begin{center}
		\begin{tabular}{lcc}
			\toprule
			\textbf{Models} & \textbf{En}$\rightarrow$\textbf{De} & \textbf{De}$\rightarrow$\textbf{En}   \\
			\midrule
        LSTM    & 23.1 & 27.0\\ 
        LSTM + MMI  & 23.3 & 27.3\\ 
        LSTM + \textbf{AMI} (w/o LNS) & 24.8 & 28.9  \\
        LSTM + \textbf{AMI} & 25.0 & 29.1  \\ \hline
        Transformer & 26.7 & 30.2 \\ 
        Transformer + MMI  & 26.9 & 30.5  \\  
        Transformer + \textbf{AMI} (w/o LNS)& 28.2 & 31.9  \\ 
        Transformer + \textbf{AMI} & \textbf{28.5} & \textbf{32.1}\\
			\bottomrule
		\end{tabular}
	\end{center}
	
	\caption{\label{tab5} Machine translation results of the BLEU scores on WMT English-German for \texttt{newstest2014}.}  
\end{table}

We present the performance of our AMI for machine translation in Table \ref{tab5}. Similar to the results of dialog generation, our AMI framework improves upon both the LSTM and Transformer, and outperforms the MMI for about 1.4 BLEU scores. We also find that the improvement from LNS is quite modest, and we conjecture this is because the NMT task has the nature that the semantics of the target are fully specified by the source.

\paragraph{Case Studies}

\begin{table}[t]
    \setlength{\abovecaptionskip}{0pt}
    \setlength{\belowcaptionskip}{-2pt}
	\begin{center}
		\begin{tabular}{p{0.95\columnwidth}}
			\toprule
			{\small Ein weiteres Radarsensor prüft, ob die Grünphase für den Fußgänger beendet werden kann. }\\ \hline
			{\small \textbf{Transformer: }Another radar sensor checks whether the green phase can be terminated for pedestrians.}\\ 
			{\small \textbf{Transformer + MMI: }Another radar sensor checks whether the green phase can be closed for the pedestrian.} \\ 
			{\small \textbf{Transformer + AMI: }Another radar sensor checks whether the green phase for pedestrians can be stopped.} \\
			{\small \textbf{Human: }An additional radar sensor checks whether the green phase for the pedestrian can be ended.}\\ 
			\bottomrule
		\end{tabular}
	\end{center}
		\caption{\label{tab6} An example of German-to-English machine translation. ``Human" denotes the original sentence in the dataset.} 
		\vskip -0.1in
\end{table}

In Table \ref{tab6}, we present an example of the German-to-English translation over different models. As we can see, all of the three models correctly translate almost all the words. However, both the Transformer and Transformer with MMI translate the goal of ``ending the green phase" to ``for pedestrian", which makes the meanings of the generated sentences totally different from the original one. However, our method produces almost the same sentence with the ground truth one except for a few synonyms.
\section{Related Works}
Estimating mutual information~\cite{bahl1986maximum,brown1987acoustic} has been comprehensively studied for many tasks such as Bayesian optimal experimental design~\cite{ryan2016review,foster2019variational}, image caption retrieval~\cite{mao2014deep}, neural networks explanation~\cite{tishby2000information,tishby2015deep,gabrie2018entropy}, \textit{etc}. However, adapting MMI to sequence modeling such as text generation is empirically nontrivial as we typically have access to discrete samples but not the underlying distributions~\cite{poole2019variational}. \citet{li2016diversity} proposed to use MMI as the objective function to address the issue of output diversity in the neural generation framework. However, they use the MI-prompting objective only for testing, while the training procedure remains the same as the standard MLE. \citet{li2016deep} addressed this problem by using deep reinforcement learning to set the mutual information as the future reward. \citet{zhang2018generating} learned a dual objective to simultaneously learn two mutual information between the forward and backward models.~\citet{ye2019jointly} proposed the dual information maximization to jointly model the dual information of two tasks. However, they optimize the backward model in the same direction with the forward model, which would limit its approach to the true posterior distribution, thus result in an unreliable reward for the forward model. 

Recent development of dual learning leverages the similar idea of enforcing forward-backward consistency in language translation~\cite{he2016dual,artetxe2018unsupervised} and image-to-image translation~\cite{zhu2017unpaired,yi2017dualgan}. They minimize the reconstruction loss of backward translation to verify and improve translation quality. Unlike these works, our backward network is trained as an auxiliary model, which aims to play against the forward network to iteratively promote generation of more meaningful text.

Adversarial training for text generation has also gained significant popularity~\cite{yu2017seqgan,li2017adversarial,zhang2017adversarial,che2017maximum,lin2017adversarial,guo2018long,chen2018adversarial}. The idea behind all these works is to use the real and synthetic data to train a classifier as the discriminator, which challenges the generator to produce text that looks more natural. However, some safe but uninformative phrases like ``I don't know" or ``can you say it again" can be judged as human-generated in many instances especially when they appear frequently in the dataset, which would be easier to learn by the discriminator. Therefore, in this paper, we use the mutual information as the reward, which is much stricter than a binary classifier because it will give positive signals only if the backward model can reconstruct the source text from the generated target text. 

\section{Conclusion}
In this paper, we introduced \textit{Adversarial Mutual Information} (AMI), a novel text generation framework that addresses a minimax game to iteratively learn and optimize the mutual information between the source and target text. The forward network in the framework is trained by playing against the backward network that aims to reconstruct the source text only if its input is in the real target distribution. We proved that our objective can be trained more stably by being transformed to the problem of Wasserstein distance minimization. The experimental results on two popular text generative tasks demonstrated the effectiveness of our framework, and we show our method has the potential to lead a tighter lower bound of the MMI problem. In future, we will attempt to explore a lower variance and more unbiased gradient estimator for the text generator in this framework and apply the AMI in multi-modality situations.
\section*{Acknowledgements}
This work was supported in part by The National Key Research and Development Program of China (Grant Nos: 2018AAA0101400), in part by The National Nature Science Foundation of China (Grant Nos: 61936006), in part by the Alibaba-Zhejiang University Joint Institute of Frontier Technologies, in part by the China Scholarship Council, and in part by U.S. Department of Energy by Lawrence Livermore National Laboratory under Contract DE-AC52-07NA27344.

\bibliography{AMI_Dialog}
\bibliographystyle{icml2020}

\newpage
~
\appendix
~
\newpage

\section{Wasserstein Distance}
As shown in the equation (9) in the regular paper, the Wasserstein distance is formulated as:
\begin{equation}
\small
\begin{aligned}
W(P_{r},  P_{\theta}) = \inf_{\gamma \in \prod (P_{r},  P_{\theta})}\mathds{E}_{(x,y) \sim \gamma} \left[ \lVert x - y \lVert \right]
\end{aligned}
\end{equation}
where $\prod (P_{r},  P_{\theta})$ is the set of all joint distributions
$\gamma (x, y)$ whose marginals are respectively $P_r$ and $P_{\theta}$. Intuitively, $\gamma (x, y)$ indicates how much “mass” must be transported from $x$ to $y$ in order to transform the distributions $P_r$ into the distribution $P_{\theta}$. The Wasserstein distance then is the ``cost" of the optimal transport plan, which is also called the Earth-Mover distance. The Wasserstein distance is proven to be much weaker than many other common distances (\textit{e.g.} JS distance) so simple sequences of probability
distributions are more likely to converge under this distance~\cite{arjovsky2017wasserstein}. In this paper, we prove that our proposed objective function is equivalent to minimizing the Wasserstein distance between the synthetic data distribution and the real data distribution.

\section{Configuration Details}
We implement our models based on the OpenNMT framework\footnote{\url{https://github.com/OpenNMT/OpenNMT-py}}~\cite{klein2017opennmt}.
\subsection{Dialog Generation}
\label{app_dialog}
The sentences are tokenized by splitting on spaces and punctuation. We set the LSTM hidden unit size to 500 and set the number of layers of LSTMs to 2 in both the encoder and the decoder. Optimization is performed using stochastic gradient descent, with an initial learning rate of 1.0. The learning rate starts decaying at the step 15000 with a decay rate of 0.95 for every 5000 steps. During training, we iteratively update the noise $\delta$, forward network and backward network respectively. The mini-batch size for the update is set at 64. We set the dropout~\cite{srivastava2014dropout} ratio as 0.3. For the Transformer, The dimension of input and output is 512, and the inner-layer has the dimension of 2048. we employ 8 parallel attention layers (or heads). We set the dropout ratio as 0.1, batch size as 64. We use the Adam optimizer~\cite{kingma2014adam} with $\beta_1 = 0.9, \beta_2 = 0.998$ and $\epsilon = 10^{-9}$. We use a maximum dialog history length of 3 to prune the irrelevant context information. The $\lambda$ in the latent noise sampling (LNS) is set as 0.1.


\subsection{Neural Machine Translation}
\label{app_nmt}
When training our NMT systems, we split the data into subword units
using BPE~\cite{sennrich2016neural}. We train 4-layer LSTMs of 1024 units with bidirectional encoder, the embedding dimension is 1024. The model is trained with stochastic gradient descent with a learning rate that began at 1. We train for 340K steps; after 170K steps, we start halving the learning rate every 17K step. During training, we iteratively update the noise $\delta$, forward network and backward network for 5k steps, 35 steps, 10k steps respectively. Our batch size is set as 128, and the dropout rate is 0.2. We adopt the default setting of the base Transformer for NMT in the OpenNMT framework. The $\lambda$ in the latent noise sampling (LNS) is set as 0.1.

\newpage
\section{Error Analysis}
In Table \ref{tab_err}, we show two examples for a representative problem of leveraging the mutual information in dialog generation. In these cases, both the MMI and AMI generate sentences about the hobbies or preferences, which have strong connection with the word ``fun" in the dialog history of the first example, and ``favorite" in the dialog history of the second example but not suitable for the given contexts. This indicates that maximizing the mutual information between the source and target may sometimes overly focuses on only a part of the sequences, thus ignores the whole semantic meanings and results in improper responses. In contrast, although the original Transformer produces a common sentence ``I do it for a living" in the first case, and an unrelated sentence ``Are you a hunter, or have you ever been" in the second case, they are the better choices compared to another two candidates in the table.

\begin{table}[bh]
    \setlength{\abovecaptionskip}{0pt}
    \setlength{\belowcaptionskip}{-2pt}
	\begin{center}
		\begin{tabular}{p{0.95\columnwidth}}
			\toprule
			{\small [A]: I teach a kindergarten class in Seattle. }\\
			{\small [B]: That must be fun, but also a lot. } \\ 
			{\small \textbf{Transformer: }I do it for a living.}\\ 
			{\small \textbf{Transformer + MMI: }Do you watch game of thrones?} \\ 
			{\small \textbf{Transformer + AMI: }Okay. Do you have any hobbies?} \\
			{\small \textbf{Human: }It is difficult but it is my passion. Passion has trouble paying the bills though.}\\ \hline \hline
			{\small [A]: What are you going shopping for? }\\
			{\small [B]: Clothes. I am a tall guy and kind of big so shopping is not my favorite. } \\ 
			{\small \textbf{Transformer: }Are you a hunter, or have you ever been?}\\ 
			{\small \textbf{Transformer + MMI: }I listen to much music. Do you play the guitar?} \\ 
			{\small \textbf{Transformer + AMI: }I do not like to clean up after my husband.} \\
			{\small \textbf{Human: }Do not like shopping much either, cause i am broke. I will be successful soon though!}\\
			\bottomrule
		\end{tabular}
	\end{center}
		\caption{\label{tab_err} Two Examples of dialog generation in PersonaChat dataset. ``Human" denotes the original response in the dataset, and all the generation models represent the role [A].}   
\end{table}

\end{document}